\documentclass{article}

\usepackage{microtype}
\usepackage{graphicx}
\usepackage{subfigure}
\usepackage{booktabs} 

\usepackage{hyperref}

\usepackage{times}  
\usepackage{helvet}  
\usepackage{courier}  
\usepackage{url}  
\usepackage{graphicx}  

\usepackage[utf8]{inputenc} 
\usepackage[T1]{fontenc}    
\usepackage{hyperref}       
\usepackage{url}            
\usepackage{booktabs}       
\usepackage{amsfonts}       
\usepackage{nicefrac}       
\usepackage{microtype}      

\usepackage{thm-restate}

\usepackage[usenames]{color}
\usepackage[usenames,dvipsnames]{xcolor}
\usepackage{ifthen}
\usepackage{mathrsfs}
\usepackage{graphicx}

\usepackage{amsmath}
\usepackage{amsthm}
\usepackage{amssymb}

\usepackage{natbib}
\usepackage{bm}
\usepackage{indentfirst}
\usepackage[hyphenbreaks]{breakurl}
\usepackage{graphicx}

\usepackage{bm}
\usepackage{amsfonts}
\usepackage{amsmath}
\usepackage{amssymb}
\usepackage{thm-restate}

\usepackage{hyperref}
\hypersetup{colorlinks=true, linkcolor=blue, citecolor=blue, anchorcolor=blue, urlcolor=blue}  

\usepackage[accepted]{icml2018}

\usepackage{algorithm}
\usepackage{algorithmic}
\newcommand{\E}{\mathbb{E}}
\newcommand{\I}{\mathbb{I}}
\newcommand{\Var}{\operatornamewithlimits{Var}}
\newcommand{\R}{\mathbb{R}}
\newcommand{\argmax}{\operatornamewithlimits{argmax}}

\usepackage[usenames,dvipsnames]{xcolor}
\usepackage{ifthen}

\newcommand{\compilehidecomments}{false}
\ifthenelse{ \equal{\compilehidecomments}{true} }{%
	\newcommand{\wei}[1]{}
	\newcommand{\siwei}[1]{}
}{
	\newcommand{\wei}[1]{{\color{blue!50!black}  [\text{Wei:} #1]}}
	\newcommand{\siwei}[1]{{\color{brown!60!black} [\text{Siwei:} #1]}}
}

\newtheorem{mycorollary}{Corollary}

\newtheorem{theorem}{Theorem}
\newtheorem{mydef}{Definition}

\newtheorem{assumption}{Assumption}
\newtheorem{proposition}{Proposition}

\newcommand{\rad}{{\it rad}}

\icmltitlerunning{Pure Exploration Bandit Problem with General Reward Functions Depending on Full Distributions}
\allowdisplaybreaks
\begin{document}

\onecolumn

\icmltitle{Pure Exploration Bandit Problem with General Reward Functions Depending on Full Distributions}



\icmlsetsymbol{equal}{*}

\begin{icmlauthorlist}
\icmlauthor{Siwei Wang}{siwei}
\icmlauthor{Wei Chen}{wei}
\end{icmlauthorlist}

\icmlaffiliation{siwei}{Tsinghua University, Beijing, China}
\icmlaffiliation{wei}{Microsoft Research, Beijing, China}

\icmlcorrespondingauthor{Siwei Wang}{email@yourdomain.edu}
\icmlcorrespondingauthor{Wei Chen}{weic@microsoft.com}


\vskip 0.3in

\begin{abstract}
	
	
In this paper, we study the pure exploration bandit model on general distribution functions, which means that the reward function of each arm depends on the whole distribution, not only its mean. We adapt the racing framework and LUCB framework to solve this problem, and design algorithms for estimating the value of the reward functions with different types of distributions. Then we show that our estimation methods have correctness guarantee with proper parameters, and obtain sample complexity upper bounds for them. Finally, we discuss about some important applications and their corresponding solutions under our learning framework.

\end{abstract}

\section{Introduction}\label{Section_Int}

Pure exploration bandit problem can be described as a game between a player (or a learner) and the environment. The player have $m$ arms to pull.
After an arm $i$ is pulled, the player receives an observation $X_i$ independently sampled from
	a fixed unknown distribution $D_i$ corresponding to arm $i$ by the environment.
There is a known reward function $H$ for the distributions $D_i$'s.
The player needs to design a strategy for selecting arms to pull and for deciding when to stop and which
	arm to output based on his observations,
	such that the output arm $i$ has the largest reward value $H(D_i)$.

In pure exploration bandit model, the player cannot guarantee to always output the correct answer
	if he only gets finite number of observations.
One type of strategies is to maximize the success probability, given the constraint on 
	the number of observations, referred to as the fixed-budget model \cite{Audibert2010Best,Bubeck2012Multiple}.
Another type of strategies is to constrain the error probability $\delta$, and 
	try to minimize the number of samples that the learner needs to obtain, referred
	to as the fixed-confidence model \cite{evendar2006action}.

Pure exploration bandit problem can be used to solve many online optimization problems, thus it gets more and more attention in recent years. Most existing studies on pure exploration bandit model
	the reward function $H$ to be the mean function, i.e., $H(D) = \E_{X\sim D}[X]$.
However, in many real-world applications, the function $H(D)$ depends on not only its mean but the entire
	distribution $D$.
One example is that $H(D)$ measures the similarity of distribution $D$ to a known distribution $G$, and we aim to find among unknown arms the one that has
	the distribution most similar to $G$.
Such reward function can be used in target selection, e.g. finding a special target in the ocean using satellite images:
	we divide the ocean into small areas, 
	and try to figure out which area contains the target. 
In each time step, the satellite can choose one of the areas to take a photo. 
Then this photo can be viewed as a random sample of a variable that follows some unknown distribution based on the geographical conditions of that area. 
Moreover, we know what the target looks like, thus we have a known probability distribution $G$ for the area that would contain the target,
	where the uncertainty in the distribution models the weather and other geographical conditions that could affect the photo image.
So what we want is to efficiently sample photos from different areas 
to find out the area whose corresponding distribution is the closest to our known distribution $G$. 	
While the mean function as the reward function $H$ has been well studied, the case of
	general reward function has not been well addressed in the literature.


In this paper, we study the novel problem setting where the reward function depends
	on the entire underlying unknown distribution $D_i$.
To solve this problem, we consider two frameworks in the fixed-confidence setting: the racing framework \cite{maron1997the,evendar2006action,kaufmann2013information}
 and LUCB framework \cite{kalyanakrishnan2012PAC}. 
 Our main contribution is 
to design proper algorithms to estimate the reward $H(D_i)$ given a set of observations from distribution $D_i$. We divide it into three types based on its continuity property, and concentrate on the case that $H$ follows Lipschitz continuity with total variation distance. 
We show that the sample complexity for using our estimation algorithm in those frameworks is upper bounded by $O(m\log(1/\delta) / \Delta^2 + c)$, where $\delta$ is the error toleration rate, $\Delta$ is the minimal gap between the maximized $H(D_i)$ and other $H(D_j)$'s, and $c$ is a constant that may depend on $\Delta,m$ but does not depend on $\delta$.  
This means that when $\delta$ tends to 0, the complexity is asymptotically optimal. 


We highlight the difference between our general model and the existing pure exploration model here. Our model concentrates on learning not only a single parameter, but the entire shape of every distribution. In the traditional pure exploration model that only concerns finding an
	arm with the best mean reward, 
	each observation of the distribution is an independent and unbiased estimator of the mean reward, thus 
	they only need to repeat this procedure and then take average 
so that the expected reward is learned precisely. 
In our model such independent and unbiased estimator may not exist. 
For example, when we want to find a distribution with minimal distance to the
	Gaussian distribution $\mathcal{N}(0,1)$, a single observation means nothing about 
	the distribution distance, and one cannot obtain an unbiased estimator of that distance based on the single observation. 
Only when we take a large set of observations into consideration, we can obtain an estimator (with small bias) and a corresponding confidence interval for that distance. 

\subsection{Related works}

Stochastic multi-armed bandit (MAB) model \cite{Berry1985Bandit,Sutton1998Reinforcement} describes the trade-off between exploitation and exploration, different from pure exploration that only concentrates on exploration. It is the origin of pure exploration bandit model, and the algorithms for
pure exploration bandit such as the racing algorithm and LUCB algorithm follow the idea of upper confidence bound algorithm \cite{gittins1989multi,Auer2002Finite} for MAB problems.

Pure exploration bandit model has been researched in many papers.
As we have mentioned, most of those researches focus on using mean value as the reward of any given distribution. Their results follow the optimal lower bound $O(m\log(1/\delta) / \Delta^2)$ \cite{Audibert2010Best}. 
Further researches concentrate on finding the best $k$ arms \cite{Kalyanakrishnan2011Efficient}, or choosing an arm set satisfying some combinatorial structures \cite{Chen2014Combinatorial}. 
Another aspect is to consider approximate results.
For example, PAC \cite{evendar2006action} considers the case of finding a single arm $i$ with $H(D_i) \ge \max_i H(D_i)- \epsilon$. To the best of our knowledge, we are the first to consider general reward functions in pure exploration bandit.

In \cite{ChenGeneral16}, the authors considered a special kind of general reward functions on Combinatorial MAB models. 
Besides the difference that they work on the cumulative regret objective while we work on the pure exploration
	objective, our setting is still very different with theirs. 
\cite{ChenGeneral16} focused on looking for the best combinatorial arm set based on some restricted reward functions, while our main concern is to consider more general distribution reward functions.
For example, we can deal with the case of finding the distribution with minimal total variation distance to a target, but they cannot achieve the same goal.

Quantile-based MAB model considers a special kind of reward functions \cite{szorenyi2015qualitative,david2016pure}. Different with classic MAB problems that care about the mean of each distribution, in quantile-based MAB model, the reward is measured by the $\tau$-quantile of each distribution. Compare with this model, our setting is still more general, since the $\tau$-quantile can be viewed as a function of the whole distribution as well.

Another similar topic is Chernoff test \cite{Chernoff1959Sequential}. 
The original Chernoff test is a game with two hypothesis $H_1,H_2$ and two unknown distributions $D_1,D_2$. The player knows under the first hypothesis, $D_1 = G$, $D_2 = G'$; while the second means $D_1 = G'$, $D_2 = G$, but he does not know which one is true.
The distributions $G$ and $G'$ are known to the player.
He needs to observe the distributions $D_1$ and $D_2$ multiple times, and make the decision that which hypothesis is correct. 
In our setting, we can choose the distribution function $H(D) = \E_{X\sim D}[\log g(X)/\log g'(X)]$, where $g$ and $g'$ are the probability density (mass) function of $G$ and $G'$. Then $H(G) - H(G') = KL(G,G') + KL(G',G)\ge 0$, where $KL$ denotes the KL-divergence between two distributions. If the output is $D_1$, we know that $H(D_1) > H(D_2)$, which means hypothesis $H_1$ is correct. 
Using our method would achieve the same sample complexity as the Chernoff test.
However, it is unclear how to extend Chernoff test to deal with the general scenario in our
	paper because our model selects arms among arbitrary unknown distributions, not known distributions.


\section{Preliminaries}

\subsection{Models and definitions}

A pure exploration bandit problem with general distribution functions can be modelled as a tuple $(A,D,H,\delta)$. $A = \{1,2,\cdots,m\}$ is the set of all arms. $D = \{D_1,\cdots,D_m\}$ is the set of corresponding probability distributions of arms in $A$. $H$ is a reward function $\Phi \to \R$, where $\Phi$ is the set of all possible probability distributions, and $\delta$ is the error probability. At each time slot, a policy $\pi$ needs to choose an arm $i(t)$ based on the previous observations, and then observe a random variable $X(t) \sim D_{i(t)}$. The random variables $\{X(\tau)\}_{\tau=1}^t$ are independent. 
The policy $\pi$ can also choose to stop the game and output a target arm whenever he wants.
We use $T_\pi$ to denote the random variable of the stopping time of policy $\pi$, and
	$S_\pi$ to be the random variable of the output arm. 
Then the goal of the player is to design a policy $\pi$ such that with high probability, it can find the arm $i$ with the maximum value $H(D_i)$, i.e., $\Pr[S_\pi \in \argmax_i H(D_i)] \ge 1 - \delta$.
Under this constraint, he wants sample complexity $T_\pi$ as small as possible.





As commonly assumed in pure exploration bandit problems,
	we assume that there is a unique optimal arm $i^*$, i.e., $H(D_{i^*}) = \max_i H(D_{i})$ and $H(D_{i^*}) > \max_{i\ne i^*} H(D_{i})$. By this assumption, we can define the gap $\Delta_i$ as following:

\begin{displaymath}
    \Delta_i = \left\{
    \begin{array}{ll}
    H(D_{i^*}) - H(D_i) & \textrm{if $i\ne i^*$}\\
    H(D_{i^*}) - \max_{j\ne i^*} H(D_j)& \textrm{if $i= i^*$}
    \end{array} \right.
\end{displaymath}

Since the optimal solution is unique, $\forall i \in A, \Delta_i > 0$.


\begin{assumption}\label{Assumption_H_con}
	There exists a constant $B$, such that for any distribution $G,G'$, we have
  \begin{equation*}
    |H(G)-H(G')| \le B \cdot \Lambda(G,G'),
  \end{equation*}
where $\Lambda(\cdot, \cdot)$ denotes some type of distribution distance, which is positive, symmetric and follows the triangle inequality.
\end{assumption}

Assumption \ref{Assumption_H_con} is to ensure the continuity of function $H$. Most functions in real applications satisfy this assumption.










\subsection{Example of applications}

In this section, we discuss some important applications of our model.

\subsubsection{Selecting the distribution with the largest $\tau$-quantile}

Quantile-based multi-armed bandits model has attracted people's attention in recent years \cite{szorenyi2015qualitative,david2016pure}. This model aims to find out the distribution with largest $\tau$-quantile, and it is widely used in applications such as clinical trials and risk assessment \cite{szorenyi2015qualitative,schachter1997irreverent}. 

\subsubsection{Selecting the closest distribution to target distribution}

Finding out a target distribution among several candidates is an important problem in hypothesis testing. People propose many frameworks such as Chernoff test and target scanning \cite{Chernoff1959Sequential,bessler1960theory,zigangirov1966problem,dragalin1996simple}, and they are widely used in quality control \cite{pochampally2014six} and medical problems \cite{larsen1976statistics}. 
Except for finding out a distribution that is identical with the target, looking for the distribution that is the closest to the target is also an important question, which can be used in applications such as target selection or target searching. Our results provide a novel solution for this problem as well.

\subsubsection{Selecting the distribution that follows the target type}

People are also interested in the question of looking for distributions that follow a special type, e.g., finding out a Gaussian distribution or an exponential distribution among several candidates.
This problem is another kind of hypothesis test and is commonly used in data verification \cite{avenhaus1996compliance,bensefia2004handwriting}.
Traditional solution chooses to use Kolmogorov-Smirnov test to check which distribution is more likely to be a Gaussian one \cite{lilliefors1967kolmogorov}. However, Kolmogorov-Smirnov test is based on the similarity of cumulative distribution functions and may fail in some special cases. Our model, on the other hand, can be used to solve this problem by finding out the distribution whose total variation distance with a Gaussian distribution is the smallest one, which can be more precise in most of the cases.

\section{Algorithmic frameworks}\label{Sec_Frame}

In this section, we present the racing framework \cite{maron1997the,evendar2006action,kaufmann2013information}
 and LUCB framework \cite{kalyanakrishnan2012PAC} proposed to solve the traditional pure exploration bandit problem. 
We make adjustments to  these frameworks in order to support general reward functions depending on the full distributions.
In particular, we replace the estimation function in these frameworks with $\textbf{Estimate}$, and abstract the sample-size function
	$n_H(\delta,\Delta)$ and the gap function $\Delta_H(\delta,n)$, which will be instantiated 
	for different $H$ functions in Section~\ref{Section_Esti}.
We prove that our adjustment does not influence the correctness guarantee if $n_H(\delta,\Delta)$ and $\Delta_H(\delta,n)$  
	are chosen appropriately. 

\subsection{Racing framework}

The racing framework is shown in Algorithm \ref{Algorithm_Elim_F}. The basic idea of the racing framework is to divide the game into several phases. In each phase, the remaining arms will be pulled for the same number of times so that the size of their confidence intervals on $H(D_i)$'s is decreased by a half. 
Then at the end of each phase, the arms with larger gap $\Delta_i$ can be eliminated, since their upper confidence bounds are smaller than the lower confidence bound of a particular arm. 

An important parameter in the racing framework is the number of times that an arm should be pulled in the $k$-th phase. In this paper, we define $n_H(\delta, \Delta)$ to be the value such that for any distribution $D$, 
\begin{equation}\label{eq_10}
\forall n \ge n_H(\delta, \Delta), \Pr[|H(D) - \textbf{Estimate}(O(D,n))| \ge \Delta] \le \delta,
\end{equation}
where $O(D,n)$ is a set of $n$ i.i.d. samples drawn from $D$. If at the end of the $k$-th phase, we want the confidence radius of the estimated $H(D_i)$ to be $\rad(k)$ and the confidence level to be $\delta(k)$, then we only need 
$n_H(\delta(k), \rad(k))$ observations for all the remaining arms. 

\begin{algorithm}[t]
    \centering
    \caption{{\sf Racing Framework}}\label{Algorithm_Elim_F}
    \begin{algorithmic}[1]
    \STATE \textbf{Input:} The set of arms $A$, the error probability $\delta$.
    \STATE \textbf{Init:} $k = 0$.
    \WHILE {$|A| > 1$}
    \STATE $k \gets k + 1$
    \STATE $\rad(k) = 2^{-k}$; $\delta(k) ={\delta \over 2|A|k^2}$.
    \STATE Pull each arm $i \in A$ until there are $n_H(\delta(k), \rad(k))$ observations from $D_i$, let $O_i(k)$ be the set of these observations.
    \STATE For each arm $i \in A$, $\hat{H}_i(k) = \textbf{Estimate}(O_i(k))$.
    \STATE For all arms $j \in A$, if there exists any $i \in A$ such that $\hat{H}_i(k) - 2\rad(k) \ge \hat{H}_j(k)$, delete arm $j$ from $A$.
    \ENDWHILE
    \STATE \textbf{Return:} The remaining arm in $A$.
    \end{algorithmic}
\end{algorithm}


\begin{restatable}{proposition}{PropOne}\label{Prop_Elim_F}
With probability at least $1-\delta$, Algorithm \ref{Algorithm_Elim_F} works correctly, and the sample complexity $T_R$ satisfies:

\begin{equation*}
T_R \le \sum_{i=1}^m n_H\left({\delta \over 2m\log^2\left({8\over \Delta_i}\right)}, {\Delta_i \over 8}\right).
\end{equation*}
\end{restatable}

\subsection{LUCB framework}

LUCB framework is shown in Algorithm \ref{Algorithm_LUCB_F}. The idea is to choose the one with the larger uncertainty between the empirically best arm and the arm that has the largest potential, 
and stop only if there exists an arm whose lower confidence bound is larger than other arms' upper confidence bounds. 

In LUCB framework, 
we define $\Delta_H(\delta,n)$ to be the value such that for any distribution $D$,
\begin{equation}\label{eq_11}
\Pr[|H(D) - \textbf{Estimate}(O(D,n))| \ge \Delta_H(\delta,n)] \le \delta.
\end{equation} 
Then in any time slot $t$, we can obtain the confidence bounds for $H(D_i)$ by $\textbf{Estimate}(O_i(t))$ and $\Delta_H(\delta(t), N_i(t))$, where $\delta(t)$ is the confidence level, and $N_i(t)$ is the number of observations on arm $i$ until time $t$.


\begin{algorithm}[t]
    \centering
    \caption{{\sf LUCB Framework}}\label{Algorithm_LUCB_F}
    \begin{algorithmic}[1]
    \STATE \textbf{Input:} The set of arms $A$, the error probability $\delta$.
    \STATE \textbf{Init:} $t = 0$, let $O_i(t)$ be the set of all observations of arm $i$ until time step $t$.
    \STATE Observe each arm once, $t \gets t+|A|$.
    \STATE For each arm $i$, $\hat{H}_i(t) = \textbf{Estimate}(O_i(t))$, $\rad_i(t) = \Delta_H({\delta \over 2mt^2},|O_i(t)|)$.
    \WHILE {\textbf{true}}
    \STATE $i_1(t) = \argmax_{i\in [m]} \hat{H}_i(t)$.
    \STATE $i_2(t) = \argmax_{i\in [m],i\ne i_1(t)} (\hat{H}_i(t) + \rad_i(t) ) $.
    \IF {$\hat{H}_{i_1(t)}(t) - \rad_{i_1(t)}(t)\ge \hat{H}_{i_2(t)}(t) + \rad_{i_2(t)}(t)$}
    \STATE \textbf{Return:} $i_1(t)$.
    \ELSE
    \STATE Pull arm $i(t) = \argmax_{i \in \{i_1(t),i_2(t)\}} \rad_i(t)$.
    \STATE $t \gets t+1$.
    \STATE For each arm $i$, $\hat{H}_i(t) = \textbf{Estimate}(O_i(t))$, $\rad_i(t) = \Delta_H( {\delta \over 2mt^2},|O_i(t)|)$.
    \ENDIF
    \ENDWHILE
    \end{algorithmic}
\end{algorithm}

\begin{restatable}{proposition}{PropTwo}\label{Prop_LUCB_F}
With probability at least $1-\delta$, Algorithm \ref{Algorithm_LUCB_F} works correctly, and the sample complexity $T_L$ satisfies:

\begin{equation*}
T_L \le \min_{t > 0} \{t > \sum_{i=1}^m n_H({\delta \over 2mt^2}, {\Delta_i\over 4})\}.
\end{equation*}
\end{restatable}




\section{Estimation Methods for Different Type of $H$ Functions}\label{Section_Esti}

In Section \ref{Sec_Frame}, we see that the key points of solving our problem are the function $\textbf{Estimate}$ and corresponding functions $n_H(\delta,\Delta), \Delta_H(\delta,n)$. 
The instantiation of these functions depend on the hardness of estimating $H$, which in turn depends on
	the hardness of the distance measure as given in Assumption \ref{Assumption_H_con}.
We consider three types of distance measures below, from easiest to hardest.
Our technical analysis will be focused on the hardest one, the total variation distance measure.


\begin{mydef}
$\Lambda_{M}$ is the distance between the means of two distributions, i.e., $\Lambda_{m}(D,D') = |\E_{X\sim D}[X] - \E_{X'\sim D'}[X']|$.

$\Lambda_{K}$ is the Kolmogorov-Smirnov distance maximum distance between two cumulative distribution functions, i.e., $\Lambda_{K}(D,D') = \max_x |F_D(x) - F_{D'}(x)|$, 
where $F_D(x)$ is the cumulative distribution function of $D$.

$\Lambda_{\it TV}$ is the total variation distance, i.e., $\Lambda_{\it TV}(D,D') = \sup_{A \subseteq S}|D(A) - D'(A)|$, 
where $S$ is the support of $D$ and $D'$, and $D(A)$ is the probability mass of set $A$ for $D$. Moreover, if the probability density functions of $D$ and $D'$ exist and are integrable, then $\Lambda_{\it TV}(D,D') = {1\over 2}\int_x |f_D(x) - f_{D'}(x)|dx$, 
where $f_D(x)$ is the probability density function of $D$.

\end{mydef}

Different distance measures lead to different complexities on estimating $H(D)$. 
Because of this, for different types of distances, we need to use different methods to estimate. 
In this section, we deal with them from easy to hard. 

\subsection{The case with $\Lambda_{m}$}

This case is almost the same as classical pure exploration bandit problem since we only need to ensure that the estimated distribution does not have a large bias on its mean. We can define the empirical distribution $\textbf{p}(O)$ as following: $p(x) = {\#(x,O)\over |O|}$ for any $x$ where $\#(x,O)$ is the number of $x$ in the observation set $O$, and then set $\textbf{Estimate}(O) = H(\textbf{p}(O))$. When the distributions are bounded by $[0,1]$, using Chernoff-Hoeffding's inequality, we know that the condition in Eq.\eqref{eq_10} holds with $n_H(\delta,\Delta) = {B^2\over 2\Delta^2} \log {2\over \delta}$ and the condition in Eq.\eqref{eq_11} holds with $\Delta_H(\delta,n) = B\sqrt{\log{2\over \delta}\over 2n}$.

\begin{proposition}\label{Proposition_LambdaM}
Similar with \cite{evendar2006action,kalyanakrishnan2012PAC}, we have that using the empirical distribution to estimate $H(D_i)$ (i.e., Algorithm \ref{Algorithm_Esti_D}) in the racing framework with  $n_H(\delta,\Delta) = {B^2\over 2\Delta^2} \log {2\over \delta}$ has sample complexity $O(\sum_{i=1}^m {B^2\over \Delta_i^2} (\log{m\over \delta}+\log\log{1\over \Delta_i}))$, and using the empirical distribution to estimate $H(D_i)$ in LUCB framework with $\Delta_H(\delta,n) = B\sqrt{\log{2\over \delta}\over 2n}$ has sample complexity $O(B^2 h \log{Bh\over \delta})$, where $h = \sum_{i=1}^m {1\over \Delta_i^2}$. 
\end{proposition}

\subsection{The case with $\Lambda_{K}$}\label{Section_LC}

The problem becomes a little harder, but the complexity does not change since we have the DKW inequality \cite{dvoretzky1956asymptotic,massart1990tight} (see details in appendix). Here the estimation method is the same, i.e., to use $\textbf{Estimate}(O) = H(\textbf{p}(O))$, where $\textbf{p}(O)$ is the empirical distribution. 
We still have that the condition in Eq.\eqref{eq_10} holds with $n_H(\delta,\Delta) = {B^2\over 2\Delta^2} \log {2\over \delta}$ and the condition in Eq.\eqref{eq_11} holds with $\Delta_H(\delta,n) = B\sqrt{\log{2\over \delta}\over 2n}$.

\begin{proposition}\label{Proposition_LambdaK}
Using the empirical distribution to estimate $H(D_i)$ (i.e., Algorithm \ref{Algorithm_Esti_D}) in the racing framework with  $n_H(\delta,\Delta) = {B^2\over 2\Delta^2} \log {2\over \delta}$ has sample complexity $O(\sum_{i=1}^m {B^2\over \Delta_i^2} (\log{m\over \delta}+\log\log{1\over \Delta_i}))$, and using the empirical distribution to estimate $H(D_i)$ in LUCB framework with $\Delta_H(\delta,n) = B\sqrt{\log{2\over \delta}\over 2n}$ has sample complexity $O(B^2 h \log{Bh\over \delta})$, where $h = \sum_{i=1}^m {1\over \Delta_i^2}$. 
\end{proposition}
%



\subsection{The case with $\Lambda_{\it TV}$}

In the first two subsections, we only need to care about the mean value and the cumulative distribution function of each distribution, which does not depend on whether the distribution itself is discrete or continuous. 
However, when we come to total variation distance, we need to construct the probability mass (density) function, which can be really different between the discrete case and the continuous case. Because of this, we divide this subsection into four parts.

\subsubsection{Discrete distribution with finite support}

First we consider the case that $D_i$'s are discrete distributions with finite support $S$. In this case, we can still use the empirical distribution (Algorithm \ref{Algorithm_Esti_D}) to estimate $H(D_i)$.

\begin{algorithm}[t]
    \centering
    \caption{{\sf Estimate } $H(D)$ \sf{ for discrete distributions with finite support}}\label{Algorithm_Esti_D}
    \begin{algorithmic}[1]
    \STATE \textbf{Input:} Observations $O$, support $S$, function $H$.
    \STATE \textbf{Init:} $\forall s \in S, x_s = 0$.
    \FOR {$o_i \in O$}
    \STATE $x_{o_i} \gets x_{o_i} + 1$
    \ENDFOR
    \STATE Construct $\textbf{p}$ such that $\forall s \in S, p_s = {x_s \over |O|}$
    \STATE \textbf{Return:} Estimated value $H(\textbf{p})$; 
    \end{algorithmic}
\end{algorithm}

\begin{theorem}\label{Theorem_DDT}
If $H$ follows Assumption \ref{Assumption_H_con} with $\Lambda_{\it TV}$ and constant $B$, and $D_i$'s are discrete distributions with finite support $S$.  Then Algorithm \ref{Algorithm_Esti_D} satisfies that: 
the condition in Eq.\eqref{eq_10} holds with $n_H(\delta, \Delta) = {B^2\over \Delta^2}(\log {1\over \delta} + {|S|\over 2})$, and  the condition in Eq.\eqref{eq_11} holds with $\Delta_H(\delta,n) = B\sqrt{\log{1\over \delta}\over 2n} + B\sqrt{|S| \over 4n}$.
\end{theorem}

\begin{proof}
From Theorem 2 in \cite{berend2012convergence}, we know that for any $\epsilon \ge \sqrt{|S|\over 4n}$, 

\begin{equation*}
\Pr[\Lambda_{\it TV}(\textbf{p}, D) \ge \epsilon] \le \exp\left(-2n\left(\epsilon - \sqrt{|S|\over 4n}\right)^2\right),
\end{equation*}
where 
$n$ is the number of observations.  

Set $\delta = \exp(-2n(\epsilon - \sqrt{|S|\over 4n})^2)$, we can find out that $\epsilon = \sqrt{\log{1\over \delta}\over 2n}  + \sqrt{|S| \over 4n}$, which means that the condition in Eq.\eqref{eq_11} holds with $\Delta_H(\delta,n) = B\sqrt{\log{1\over \delta}\over 2n} + B\sqrt{|S| \over 4n}$.

On the other hand, we can get $n \le {1\over \epsilon^2}(\log {1\over \delta} + {|S|\over 2})$, which means that $n_H(\delta, \Delta) = {B^2\over \Delta^2}(\log {1\over \delta} + {|S|\over 2})$ is enough to make the condition in Eq.\eqref{eq_10} holds. \hspace*{\fill}$\qed$
\end{proof}

By Theorem \ref{Theorem_DDT}, we have the following corollary.

\begin{mycorollary}\label{Proposition_1}
With $\Delta_H(\delta,n), n_H(\delta, \Delta)$ set as in Theorem \ref{Theorem_DDT}, using Algorithm \ref{Algorithm_Esti_D} in the racing framework has sample complexity $O(\sum_{i=1}^m {B^2\over \Delta_i^2} (\log{m\over \delta}+\log\log{1\over \Delta_i})) + O(\sum_{i=1}^m {B^2|S| \over \Delta_i^2})$, and using Algorithm \ref{Algorithm_Esti_D} in LUCB framework has sample complexity $O(B^2h \log{Bh\over \delta}) + O(\sum_{i=1}^m {B^2|S| \over \Delta_i^2})$, where $h = \sum_{i=1}^m {1\over \Delta_i^2}$. 
\end{mycorollary} 

We can see that there is only a constant gap $O(\sum_{i=1}^m {B^2|S| \over \Delta_i^2})$ from the classical pure exploration complexity (as stated in Proposition \ref{Proposition_LambdaM}), and the gap does not depend on $\delta$. 


\subsubsection{Discrete distribution with infinite support}


Now we consider the case that the discrete distributions have infinite but countable support. Without loss of generality, we assume that the support is $\mathbb{N} = \{0,1,2,\cdots, \}$. 
Moreover, in this case there must be a bounded interval such that most of the probability mass is in. 
To make this idea standard, we use the following assumption:

\begin{assumption}\label{Assumption_Exists_DUn}
There exists constant $\beta,\lambda$ such that for any $i$ and $z \in \mathbb{N}$, $d_i(z) \le \beta\exp(-\lambda z)$, where $d_i$ is the probability mass function of $D_i$. 
\end{assumption}

In this case, we can also use the empirical distribution (i.e., Algorithm \ref{Algorithm_Esti_D}) to estimate $H(D_i)$.

\begin{theorem}\label{Theorem_DDTC}
If $H$ follows Assumption \ref{Assumption_H_con} with $\Lambda_{\it TV}$ and constant $B$, and $D_i$'s are discrete distributions with infinite support $\mathbb{N}$. Then under Assumption \ref{Assumption_Exists_DUn}, Algorithm \ref{Algorithm_Esti_D} satisfies that: 
the condition in Eq.\eqref{eq_10} holds with 
\begin{equation*}
n_H(\delta, \Delta) = {32B^2\log {1\over \delta} \over \Delta^2} + {16B^2 \over \lambda\Delta^2}\log{16B^2\beta^2\lambda \over (1-e^{-\lambda})^2\Delta^2},
\end{equation*}
and the condition in Eq.\eqref{eq_11} holds with 
\begin{equation*}
\Delta_H(\delta,n) = 2B\sqrt{2\log{1\over \delta} \over n} + B\sqrt{{2\over \lambda n}\log{2\beta^2\lambda n \over (1 - e^{-\lambda})^2}}.
\end{equation*}
\end{theorem}

\begin{proof}

Denote $z = {1\over \lambda} \log{2\beta  \over \epsilon(1 - e^{-\lambda})}$, then we have that 
\begin{equation*}
\sum_{z'=z}^\infty d_i(z') \le \sum_{z'=z}^\infty \beta\exp(-\lambda z') =\beta\exp(-\lambda z)\sum_{z' = z}^\infty e^{-(z'-z)\lambda} = {\epsilon \over 2}.
\end{equation*} 

For a set of observations $O$, let $\bm{p}$ denote the empirical distribution, then \begin{eqnarray*}
\lambda_{TV}(\bm{p}, D_i) &=& {1\over 2} \sum_{x \ge 0} |p(x) - d_i(x)|\\
&=& {1\over 2}\left(\sum_{x < z}|p(x) - d_i(x)| + \sum_{x \ge z} |p(x) - d_i(x)|\right)\\
&\le& {1\over 2}\left(\sum_{x < z}|p(x) - d_i(x)| + \sum_{x \ge z} |p(x) + d_i(x)| \right)\\
&\le& {1\over 2}\left(\sum_{x < z}|p(x) - d_i(x)| + \sum_{x \ge z} (p(x) + d_i(x)) \right)\\
&=&{1\over 2}\left(\sum_{x < z}|p(x) - d_i(x)| + \sum_{x \ge z} (p(x) - d_i(x) + 2d_i(x))\right) \\
&=& {1\over 2}\left(\sum_{x < z}|p(x) - d_i(x)| + \sum_{x \ge z} (p(x) - d_i(x)) + 2 \sum_{x \ge z}d_i(x)\right) \\
&\le&{1\over 2}\left(\sum_{x < z}|p(x) - d_i(x)| + \sum_{x \ge z} (p(x) - d_i(x)) + \epsilon \right)\\
&=&{1\over 2}\left(\sum_{x < z}|p(x) - d_i(x)| + |\sum_{x \ge z} (p(x) - d_i(x))|\right) + {\epsilon \over 2}. 
\end{eqnarray*}

Therefore, 
\begin{eqnarray*}
\Pr[\lambda_{TV}(\bm{p}, D_i) \ge \epsilon] &\le & \Pr\left[{1\over 2}\left(\sum_{x < z}|p(x) - d_i(x)| + |\sum_{x \ge z} (p(x) - d_i(x))|\right) + {\epsilon\over 2} \ge \epsilon\right]\\
&= & \Pr\left[{1\over 2}\left(\sum_{x < z}|p(x) - d_i(x)| + |\sum_{x \ge z} (p(x) - d_i(x))|\right)  \ge {\epsilon\over 2}\right].
\end{eqnarray*}

On the other hand, ${1\over 2}\left(\sum_{x < z}|p(x) - d_i(x)| + |\sum_{x \ge z} (p(x) - d_i(x))|\right)$ can be viewed as the total variation distance between $\bm{p}$ and $D_i$ when we regard all the elements $x \ge z$ as a new element. Thus by Theorem 2 in \cite{berend2012convergence}, we have
\begin{eqnarray*}
\Pr[\lambda_{TV}(\bm{p}, D_i) \ge \epsilon]& \le &\Pr\left[{1\over 2}\left(\sum_{x < z}|p(x) - d_i(x)| + |\sum_{x \ge z} p(x) - d_i(x)|\right) \ge {\epsilon\over 2}\right] \\
&\le& \exp\left(-2n\left({\epsilon\over 2} - \sqrt{z\over 4n}\right)^2\right).
\end{eqnarray*}

Therefore, if $\epsilon \ge \sqrt{2\log{1\over \delta} \over n} + \sqrt{z\over n}$, we must have that $\Pr[\lambda_{TV}(\bm{p}, D_i) \ge \epsilon] \le  \delta$. Note that $z$ is a function on $\epsilon$ ($z = {1\over \lambda} \log{2\beta  \over \epsilon(1 - e^{-\lambda})}$), hence we consider $\epsilon_1, \epsilon_2$ such that ${\epsilon_1 \over 2} \ge \sqrt{2\log{1\over \delta} \over n}$ and ${\epsilon_2 \over 2} \ge \sqrt{ {1\over n\lambda} \log{2\beta  \over \epsilon_2(1 - e^{-\lambda})}}$. In this case $\epsilon = \epsilon_1 + \epsilon_2$ satisfies that
\begin{eqnarray*}
{\epsilon \over 2}={\epsilon_1 + \epsilon_2 \over 2}\ge{\epsilon_1 \over 2}\ge \sqrt{2\log{1\over \delta} \over n},
\end{eqnarray*}
and
\begin{eqnarray*}
{\epsilon \over 2}={\epsilon_1 + \epsilon_2 \over 2}\ge{\epsilon_2 \over 2}\ge \sqrt{ {1\over n\lambda} \log{2\beta  \over \epsilon_2(1 - e^{-\lambda})}}\ge\sqrt{ {1\over n\lambda} \log{2\beta  \over \epsilon(1 - e^{-\lambda})}},
\end{eqnarray*}
which implies that $\epsilon \ge \sqrt{2\log{1\over \delta} \over n} + \sqrt{z\over n}$.

After some basic computation, we have that $\epsilon_1 = 2\sqrt{2\log{1\over \delta} \over n}$ and $\epsilon_2 = \sqrt{{2\over \lambda n}\log{2\beta^2\lambda n \over (1 - e^{-\lambda})^2}}$, i.e., the condition in Eq.\eqref{eq_11} holds with $\Delta_H(\delta,n) = 2B\sqrt{2\log{1\over \delta} \over n} + B\sqrt{{2\over \lambda n}\log{2\beta^2\lambda n \over (1 - e^{-\lambda})^2}}$.

To bound $n_H(\delta, \Delta)$, we use a similar trick, i.e., we define two values $n_1,n_2$ as following: $2B\sqrt{\log{1\over \delta}\over 2n_1} \le {\Delta\over 2}$, $B\sqrt{{2\over \lambda n_2}\log{2\beta^2\lambda n_2 \over (1 - e^{-\lambda})^2}} \le {\Delta\over 2}$. Then we know that
\begin{eqnarray*}
&&2B\sqrt{2\log{1\over \delta} \over (n_1+n_2)} + B\sqrt{{2\over \lambda (n_1+n_2)}\log{2\beta^2\lambda (n_1+n_2) \over (1 - e^{-\lambda})^2}}\\
& \le& 2B\sqrt{\log{1\over \delta}\over 2n_1} +B\sqrt{{2\over \lambda n_2}\log{2\beta^2\lambda n_2 \over (1 - e^{-\lambda})^2}} \\
&\le& {\Delta\over 2} + {\Delta\over 2}\\
&=& \Delta,
\end{eqnarray*}
which means that we can use $n_1+n_2$ as an upper bound of $n_H(\delta, \Delta)$. After some basic computation, we have that $n_1 = {32B^2\log {1\over \delta} \over \Delta^2}$ and $n_2 = {16B^2 \over \lambda\Delta^2}\log{16B^2\beta^2\lambda \over (1-e^{-\lambda})^2\Delta^2}$. \hspace*{\fill}$\qed$
\end{proof}

By Theorem \ref{Theorem_DDTC}, we have the following corollary.

\begin{mycorollary}\label{Proposition_1TC}
With $\Delta_H(\delta,n), n_H(\delta, \Delta)$ set as in Theorem \ref{Theorem_DDTC}, using Algorithm \ref{Algorithm_Esti_D} in the racing framework has sample complexity $O(\sum_{i=1}^m {B^2\over \Delta_i^2} (\log{m\over \delta}+\log\log{1\over \Delta_i})) + O(\sum_{i=1}^m  {B^2 \over \lambda\Delta_i^2}\log{B\beta\lambda \over (1-e^{-\lambda})\Delta_i})$, and using Algorithm \ref{Algorithm_Esti_D} in LUCB framework has sample complexity $O(B^2h\log{Bh\over \delta}) + O(\sum_{i=1}^m  {B^2 \over \lambda\Delta_i^2}\log{B\beta\lambda \over (1-e^{-\lambda})\Delta_i})$, where $h = \sum_{i=1}^m {1\over \Delta_i^2}$.
\end{mycorollary} 

We can see that there is only a constant gap $O(\sum_{i=1}^m  {B^2 \over \lambda\Delta_i^2}\log{B\beta\lambda \over (1-e^{-\lambda})\Delta_i})$ from the classical pure exploration complexity (as stated in Proposition \ref{Proposition_LambdaM}), and the gap does not depend on $\delta$. 


\subsubsection{Continuous distribution with bounded support}

Next we consider the case that the distributions are continuous and bounded. Without loss of generality, we suppose the support is $[0,1]$. Since the observations are discrete, we cannot deal with non-continuous probability density functions. Because of this, we need the following two assumptions.

\begin{algorithm}[t]
    \centering
    \caption{{\sf Estimate } $H(D)$ \sf{ for bounded continuous distributions}}\label{Algorithm_Esti_C}
    \begin{algorithmic}[1]
    \STATE \textbf{Input:} Observations $O = \{o_1,o_2,\cdots\}$, function $H$, constant $C$.
    \STATE \textbf{Init:} $\ell = {1\over 2}$, $r = {1\over \ell}$.
    \STATE Let $I_1 = [0,\ell],\cdots,I_r = (1-\ell, 1]$; $x_1 = x_2 = \cdots = x_r = 0$.
    \FOR {$n = 1,2,\cdots$}\label{line:6} 
    \WHILE {$\sqrt{1 \over 4n\ell} < {C\ell\over 4}$}
    \STATE $\ell \gets \ell / 2$
    \STATE $r \gets 1 / \ell$\label{line:8}
    \ENDWHILE
    \IF {$\ell$ is changed}
    \STATE Let $I_1 = [0,\ell],\cdots,I_r = (1-\ell, 1]$; $x_1 = x_2 = \cdots = x_r = 0$.
    \FOR {$i\le n$}
    \STATE If $o_i \in I_s$, $x_s = x_s+1$.
    \ENDFOR
    \ELSE
    \STATE If $o_n \in I_s$, $x_s = x_s+1$.
    \ENDIF
    \ENDFOR
     \STATE Construct $\textbf{p}$ such that $\forall z \in I_s, p(z) = {1\over \ell}\cdot {x_s \over |O|}$.
    \STATE \textbf{Return:} Estimated value $H(\textbf{p})$.
    \end{algorithmic}
\end{algorithm}
\begin{assumption}\label{Assumption_Exists}
  The probability density function exists for every $D_i$, which is $d_i$. $d_i$ is integrable on the support $[0,1]$.
\end{assumption}


\begin{assumption}\label{Assumption_Lip}
  There exists a constant C such that
  \begin{equation*}
    \forall i \in \{1,\cdots,m\}, \forall x,y , |d_i(x)-d_i(y)| \le 2C|x-y|.
  \end{equation*}
\end{assumption}

Assumption \ref{Assumption_Exists} states that the probability density functions (of all the distributions $D_i$'s) exist and Assumption \ref{Assumption_Lip} ensures the continuity of these probability density functions. Based on these two assumptions, we can obtain the following corollary, and estimate the value of $H(D_i)$ by Algorithm \ref{Algorithm_Esti_C}.



\begin{mycorollary}\label{Prop_C}
Under Assumptions \ref{Assumption_Exists} and \ref{Assumption_Lip}, for all $i \in A$, for all $0 \le x \le y \le 1$, we have
  \begin{equation*}
    {1\over 2}\int_x^y |d_i(z)-\mu(x,y,d_i)|dz \le {C(x-y)^2\over 4},
  \end{equation*}
where $\mu(x,y,d_i) = {\int_x^y d_i(z)dz \over y-x}$ is the average density of distribution $d_i$ in interval $[x,y]$ or $(x,y]$.
\end{mycorollary}

The basic idea of Algorithm \ref{Algorithm_Esti_C} is to build a discrete distribution for simulating a continuous one, i.e. divide the continuous distribution into several intervals, and treat each interval as an element in the support of the discrete distribution. If all the intervals have size $\ell$, then there are ${1\over \ell}$ such intervals. 
The difficult point is that we do not know how many intervals is enough. If we choose a small $\ell$, then the support of that discrete distribution becomes too large, which leads to a high complexity. But if the number of intervals is too small, then there will be a large gap between $D_i$ and the estimated distribution within each interval. From Corollary \ref{Prop_C}, we know that the difference between a curve and several horizontal line segments are upper bounded by ${1\over \ell} \cdot {C\ell^2 \over 4} = {C\over 4}\ell$. This means that such difference decreases when we shrink the interval size $\ell$. Then we can use an adaptive method to find a proper length $\ell$: only if $\ell$ is so large, i.e., ${C\ell \over 4} \ge \sqrt{1\over 4n\ell}$, we shrink the size of $\ell$ to a half. This makes sure that the final $\ell$ is neither too large nor too small. 

\begin{theorem}\label{Theorem_CDT}
If $H$ follows Assumption \ref{Assumption_H_con} with $\Lambda_{\it TV}$ and constant $B$, and $D_i$'s are bounded continuous distributions that follow Assumptions \ref{Assumption_Exists} and \ref{Assumption_Lip}, then Algorithm \ref{Algorithm_Esti_C} satisfies that: 
the condition in Eq.\eqref{eq_10} holds with
\begin{equation*}
n_H(\delta,\Delta) = {2B^2\log {1\over \delta} \over \Delta^2} + {8\sqrt{2}B^3C \over \Delta^3},
\end{equation*}
and the condition in Eq.\eqref{eq_11} holds with
\begin{equation*}
\Delta_H(\delta,n) =  B\sqrt{\log{1\over \delta}\over 2n} + B\left({\sqrt {2}C \over n}\right)^{1\over 3}.
\end{equation*}
\end{theorem}

\begin{proof}
Let $d_i(\ell,s) \triangleq \int_{(s-1)\ell}^{s\ell} d_i(z)dz$ be the probability mass for distribution $D_i$ in interval $((s-1)\ell(k),s\ell(k)]$. Then we can bound the gap $\Lambda_{\it TV}(\textbf{p},D_i)$ as following:
\vskip -5mm
\begin{eqnarray}
\nonumber \Lambda_{\it TV}(\textbf{p},D_i) &=& {1\over 2} \sum_{s=1}^{1\over \ell} \int_{(s-1)\ell}^{s\ell} |p(z) - d_i(z)| dz\\
\nonumber&\le& {1\over 2} \sum_{s=1}^{1\over \ell} \int_{(s-1)\ell}^{s\ell}  \left|p(z) - {d_i(\ell,s) \over \ell} \right|dz\\
\nonumber &&+{1\over 2} \sum_{s=1}^{1\over \ell}  \int_{(s-1)\ell}^{s\ell} \left|{d_i(\ell,s) \over \ell} - d_i(z)\right| dz\\
\label{eq:112}
&\le & {1\over 2} \sum_{s=1}^{1\over \ell} |p_s - d_i(\ell,s)| + {C\ell \over 4},
\end{eqnarray}
where Inequality \eqref{eq:112} is given by Corollary \ref{Prop_C}.

Similar with the proof of Theorem \ref{Theorem_DDT}, the first term in \eqref{eq:112} can be upper bounded by $\sqrt{\log{1\over \delta}\over 2n} + \sqrt{1 \over 4n\ell}$ with probability at least $1-\delta$. As for the second term, in lines 5-8 of Algorithm \ref{Algorithm_Esti_C}, we know that ${C\ell \over 4} \le \sqrt{1\over 4n\ell}$. Thus, $\Lambda_{\it TV}(\textbf{p},D_i) \le \sqrt{\log{1\over \delta}\over 2n} + \sqrt{1 \over n\ell}$ with probability at least $1-\delta$.



Notice that $\ell$ is not an input in $\Delta_H$ or $n_H$, thus we still need a bound on it. This is given by lines 5-8 as well. Before shrinking the size of $\ell$, $2\ell$ satisfies that $\sqrt{1\over 4n(2\ell)} \le {C(2\ell) \over 4}$, which implies $\ell \ge \left({1\over 2C^2n}\right)^{1\over 3}$. 

Then we have $\Delta_H(\delta,n) = B(\sqrt{\log{1\over \delta}\over 2n} + ({\sqrt {2}C \over n})^{1\over 3})$.

To bound $n_H(\delta, \Delta)$, we define two values $n_1,n_2$ as following: $B\sqrt{\log{1\over \delta}\over 2n_1} = {\Delta\over 2}$, $B({\sqrt {2}C \over n_2})^{1\over 3} = {\Delta\over 2}$. Then we know that $B(\sqrt{\log{1\over \delta}\over 2(n_1+n_2)} + ({\sqrt {2}C \over n_1+n_2})^{1\over 3}) \le B\sqrt{\log{1\over \delta}\over 2n_1}+B({\sqrt {2}C \over n_2})^{1\over 3} = \Delta$, which means that we can use $n_1+n_2$ as an upper bound of $n_H(\delta, \Delta)$. After some basic computation, we have that $n_1 = {2B^2\log {1\over \delta} \over \Delta^2}$ and $n_2 = {8\sqrt{2}B^3C \over \Delta^3}$.         \hspace*{\fill}$\qed$
\end{proof}

By Theorem \ref{Theorem_CDT}, we have the following corollary.

\begin{mycorollary}\label{Proposition_2}
With $\Delta_H(\delta,n), n_H(\delta, \Delta)$ set as in Theorem \ref{Theorem_CDT}, using Algorithm \ref{Algorithm_Esti_C} in the racing framework has time complexity $O(\sum_{i=1}^m {B^2\over \Delta_i^2} (\log{m\over \delta}+\log\log{1\over \Delta_i})) + O(B^3C\sum_{i=1}^m {1\over \Delta_i^3}) $, and using Algorithm \ref{Algorithm_Esti_C} in LUCB framework has time complexity $O(B^2h\log{Bh\over \delta}) + O( B^3C\sum_{i=1}^m {1\over \Delta_i^3})$, where $h = \sum_{i=1}^m {1\over \Delta_i^2}$. 
\end{mycorollary} 

Compare with Corollary \ref{Proposition_1}, the additive constant in  Corollary \ref{Proposition_2} is $O(B^3C\sum_{i=1}^m {1\over \Delta_i^3})$ (which is independent with $\delta$ too). The reason is that when applying Algorithm \ref{Algorithm_Esti_C} in the Racing Framework or the LUCB Framework, the game stops when $\ell_i = \Theta({\Delta_i \over BC})$. Thus, the estimated discrete distribution for arm $i$ has support size $\Theta({BC \over \Delta_i})$, which means that its corresponding additive constant term becomes $O({B^2|S_i|\over \Delta_i^2}) = O({B^3C\over \Delta_i^3})$.




\subsubsection{Continuous distribution with unbounded support}

%

Now we consider the most complicated case: $D_i$'s are unbounded continuous distributions. We still need to assume that the probability density function exists and integrable for every distribution $D_i$. 
Moreover, there must be a bounded interval such that most of the probability mass is in. 
To make this idea standard, we use the following assumption:

\begin{assumption}\label{Assumption_Exists_Un}
  The probability density functions $d_i$'s exist for every $D_i$'s, and they are integrable. There exists constant $\beta,\lambda$ such that for any $i$ and $z$, $d_i(z) \le \beta\exp(-\lambda |z|)$.
\end{assumption}

Compare with the bounded support case, here we need to consider two variables for simulating the discrete distribution: the small interval size $\ell$, and the bounded interval $[-L,L]$. The large interval $[-L,L]$ is divided into small intervals with length $\ell$, and the remaining parts of the support ($(-\infty,-L]$ and $(L,\infty)$) are not divided since from Assumption \ref{Assumption_Exists_Un} we know that there are not so many probability mass in them. The complete estimation method is shown in Algorithm \ref{Algorithm_Esti_U}.

\begin{algorithm}[t]
    \centering
    \caption{{\sf Estimate $H(D)$ for unbounded continuous distributions}}\label{Algorithm_Esti_U}
    \begin{algorithmic}[1]
    \STATE \textbf{Input:} Observations $O$, function $H$, constant $C, \beta,\lambda$.
    \STATE \textbf{Init:} $\ell = {1\over 2}$, $L = 1$, $r = {2L \over \ell}$.
    \STATE Let $I_1 = [-L,-L+\ell],\cdots,I_r = (L-\ell, L]$, $x_1 = x_2 = \cdots = x_r = 0$.
    \STATE Let $I_- = (-\infty, -L), I_+ = (L,\infty)$, $x_- = x_+ = 0$.
    \FOR {$o_n \in O$}
    \WHILE {$\sqrt{\log{1\over \delta} \over 2n} < {2\beta \over \lambda} e^{-\lambda L}$}
    \STATE $L \gets L \cdot 2$
    \ENDWHILE
    \WHILE {$\sqrt{L \over n\ell} < {CL\ell\over 2}$}
    \STATE $\ell \gets \ell / 2$
    \ENDWHILE
    \STATE $r = 2L / \ell$
    \IF {$\ell$ or $L$ is changed}
    \STATE Let $I_1 = [-L,-L+\ell],\cdots,I_r = (L-\ell, L]$, $x_1 = x_2 = \cdots = x_r = 0$.
    \STATE Let $I_- = (-\infty, -L), I_+ = (L,\infty)$, $x_- = x_+ = 0$
    \FOR {$i\le n$}
    \STATE If $o_i \in I_s$, $x_s = x_s+1$ ($s \in \{-,1,\cdots,r,+\}$).
    \ENDFOR
    \ELSE
    \STATE If $o_n \in I_s$, $x_s = x_s+1$ ($s \in \{-,1,\cdots,r,+\}$).
    \ENDIF
    \ENDFOR
    \STATE Construct $\textbf{p}$ such that for all $1\le s\le r$, $\forall z \in I_s, p(z) = {1\over \ell} \cdot {x_s \over |O|}$; for $z > L$, $p(z) = {x_+ \over |O|} \lambda e^{-\lambda(z-L)}$; for $z < -L$, $p(z) = {x_- \over |O|}  \lambda e^{\lambda(z+L)}$.
    \STATE \textbf{Return:} Estimated value $H(\textbf{p})$.
    \end{algorithmic}
\end{algorithm}

\begin{theorem}\label{Theorem_CDUT}
If $H$ follows Assumption \ref{Assumption_H_con} with $\Lambda_{\it TV}$ and constant $B$, and $D_i$'s are unbounded continuous distributions that follow Assumption \ref{Assumption_Lip} and \ref{Assumption_Exists_Un}, then Algorithm \ref{Algorithm_Esti_U} satisfies that: 
the condition in Eq.\eqref{eq_10} holds with
\vskip -4mm
\begin{equation*}
n_H(\delta, \Delta) = {8B^2\log {1\over \delta} \over \Delta^2}+ {128B^3C \over \lambda^2\Delta^3} \log^2 {1024\beta^2 B^3C\over \lambda^4\Delta^3},
\end{equation*}  and  the condition in Eq.\eqref{eq_11} holds with
\vskip -4mm
\begin{equation*}
\Delta_H(\delta,n) = B\sqrt{2\log{1\over \delta}\over n}  + B\left({{8\sqrt{2}}C \log^2({\beta\over \lambda}\sqrt{8n\over \log{(1/\delta)}})\over n\lambda^2}\right)^{1/3}.
\end{equation*}
\end{theorem}
\begin{proof}

Similar with the proof of Theorem \ref{Theorem_CDT}, we set $d_i(L,\ell,s) = \int_{-L+(s-1)\ell}^{-L+s\ell} d_i(z) dz$ be the probability mass for distribution $D_i$ in interval $(-L+(s-1)\ell, -L+s\ell]$, while $d_i(L,\ell,-) = \int_{-\infty}^{-L} d_i(z)dz$ and $d_i(L,\ell,+) = \int_{L}^\infty d_i(z)dz$.

First consider $\int_{-\infty}^{-L} |p(z) - d_i(z)|dz$, we have:
\vskip -3mm
  \begin{eqnarray*}
    \int_{-\infty}^{-L} |p(z) - d_i(z)|dz &\le& \int_{-\infty}^{-L} (p(z) + d_i(z))dz \\
    &=& \int_{-\infty}^{-L} (2d_i(z) + p(z) - d_i(z))dz \\
    &\le&2 d_i(L,\ell,-)  + |p_- -d_i(L,\ell,-) |.
  \end{eqnarray*}

  Similarly, $\int_{L}^{\infty} |p(z) - d_i(z)|dz \le 2 d_i(L,\ell,+)  + |p_+ -d_i(L,\ell,+)|$. Thus, we have that 
  \vskip -3mm
  \begin{eqnarray}
    \nonumber\Lambda_{\it TV}(\textbf{p},D_i) &\le& d_i(L,\ell,-) +  d_i(L,\ell,+)  + {CL\ell \over 2} \\
   \label{eq_113}&&+{1\over 2}\sum_{s\in \{-,1,\cdots,{2L\over \ell},+\}}|p_s - d_i(L,\ell,s)| \\
    \nonumber&\le& {2\beta \over \lambda}e^{-\lambda L} + {CL\ell \over 2}  \\
    \label{eq_114}&&+ {1\over 2}\sum_{s\in \{-,1,\cdots,{2L\over \ell},+\}}|p_s - d_i(L,\ell,s)|.
  \end{eqnarray}
  
The proof of inequality \eqref{eq_113} is the same as that one of inequality \eqref{eq:112},
 which needs to use Corollary \ref{Prop_C}. Inequality \eqref{eq_114} is given by Assumption \ref{Assumption_Exists_Un}.
  
  By lines 6-8 and 9-11 in Algorithm \ref{Algorithm_Esti_U}, we always have ${2\beta \over \lambda}e^{-\lambda L}  \le \sqrt{\log{1\over \delta}\over 2n}$ and ${CL\ell \over 2} \le \sqrt{L \over n\ell}$. Thus,
  \vskip -3mm
  \begin{eqnarray*}
  {2\beta \over \lambda}e^{-\lambda L} + {CL\ell \over 2} \le \sqrt{\log{1\over \delta}\over 2n} + \sqrt{L \over n\ell}.
  \end{eqnarray*}
  
  By results in \cite{berend2012convergence}, the third term in inequality \eqref{eq_114} has high probability upper bound \begin{eqnarray*}   \sqrt{\log{1\over \delta}\over 2n} + \sqrt{{1 \over 4n}\cdot\left({2L \over \ell} + 2\right)} \le \sqrt{\log{1\over \delta}\over 2n} + \sqrt{L \over n\ell}. \end{eqnarray*}
  
  Thus $\Lambda_{\it TV}(\textbf{p},D_i) \le \sqrt{2\log{1\over \delta}\over n} + \sqrt{4L \over n\ell}$ with probability $1 - \delta$.
  
  Now we consider $L$ and $\ell$. From lines 6-8 in Algorithm \ref{Algorithm_Esti_U}, we know that ${2\beta \over \lambda}e^{-\lambda{L\over 2}} > \sqrt{\log{1\over \delta}\over 2n}$, which implies $L < {2\over \lambda } \log({\beta\over \lambda}\sqrt{8n\over \log{1\over \delta}})$. From lines 9-11, we know that $\sqrt{L\over n(2\ell)} < {CL(2\ell) \over 2}$, which implies $\ell > ({1\over 2C^2nL})^{1\over 3}$.

  
  
  Then $\Delta_H(\delta,n) = B\left(\sqrt{2\log{1\over \delta}\over n}  + ({{8\sqrt{2}}C \log^2({\beta\over \lambda}\sqrt{8n\over \log{(1/\delta)}})\over n\lambda^2})^{1/3}\right) $. As for $n_H(\delta, \Delta)$, we still choose $n_1, n_2$ such that $B\sqrt{2\log{1\over \delta}\over n_1} = {\Delta\over 2}$ and $B({{8\sqrt{2}}C \log^2({\beta\over \lambda}\sqrt{8n_2\over \log{(1/\delta)}})\over n_2\lambda^2})^{1/3} = {\Delta\over 2}$. This implies $n_H(\delta, \Delta) = {8B^2\log {1\over \delta} \over \Delta^2}+ {128B^3C \over \lambda^2\Delta^3} \log^2 {1024\beta^2 B^3C\over \lambda^4\Delta^3}$. \hspace*{\fill}$\qed$

\end{proof}

\begin{mycorollary}\label{Proposition_3}
With $\Delta_H(\delta,n), n_H(\delta, \Delta)$ set as in Theorem \ref{Theorem_CDUT}, using Algorithm \ref{Algorithm_Esti_U} in the racing framework has time complexity $O(\sum_{i=1}^m {B^2\over \Delta_i^2} (\log{m\over \delta}+\log\log{1\over \Delta_i})) + O(\sum_{i=1}^m {B^3C \over \lambda^2\Delta_i^3} \log^2 {\beta BC\over \lambda\Delta_i}) $, and using Algorithm \ref{Algorithm_Esti_U} in LUCB framework has complexity $O(B^2h \log{Bh\over \delta}) + O(\sum_{i=1}^m {B^3C \over \lambda^2\Delta_i^3} \log^2 {\beta BC\over \lambda\Delta_i})$, where $O(\sum_{i=1}^m {B^3C \over \lambda^2\Delta_i^3} \log^2 {\beta BC\over \lambda\Delta_i})$ is a constant gap that does not depend on $\delta$ and $h = \sum_{i=1}^m {1\over \Delta_i^2}$.
\end{mycorollary} 










\section{Applications of Our Solutions}

\subsection{Selecting the distribution with the largest $\tau$-quantile}

In this case, we can set $H(D_i)$ as the $\tau$-quantile of distribution $D_i$, then our model can be used to solve this problem. 

Constraint by the fact that we separate the estimation function and $rad(k)$ in the Racing Framework, it is impossible to use the $(\tau-rad(k))$-quantile and $(\tau+rad(k))$-quantile of the empirical distribution as the confidence bounds (as the authors did in \cite{szorenyi2015qualitative,david2016pure}). However, a slight modification on the framework (e.g., generalize the function \textbf{Estimate} to output the confidence interval) will allow us to use the same technique. In this case our solutions are reduced to their ones in the quantile setting.


\subsection{Selecting the closest distribution to target distribution}

In this case, we can set $H(D_i) = -\Lambda(D_i, G)$, where $G$ is the target distribution, and $\Lambda$ denote some special type of distance between distributions. 
For different kinds of distributions, we can use different algorithms in Section \ref{Section_Esti} and their corresponding $n_H(\delta, \Delta)$'s or $\Delta_H(\delta,n)$'s to solve this problem (with either racing framework or LUCB framework). 
Compare with prior works, our solutions have the advantage that we can deal with the case of looking for the closest distribution under the measure of total variation distance $\Lambda_{\it TV}$, while existing algorithms (e.g., Kolmogorov-Smirnov test) only works for the Kolmogorov-Smirnov distance $\Lambda_{K}$.

\subsection{Selecting the distribution that follows the target type}

Here we use Gaussian distribution as an example. 
In this case, we can set $H(D_i) = -\Lambda_{\it TV}(D_i, \mathcal{N}(\mu_i, \sigma^2_i))$, where $\mathcal{N}(\mu_i, \sigma^2_i)$ denotes the Gaussian distribution with mean and variance the same as distribution $D_i$ (i.e., $\mu_i = \E[D_i]$ and $\sigma^2_i = \Var(D_i)$) and $\Lambda_{\it TV}$ denote the total variation distance between two distributions. 
To estimate $H(D_i)$, we can first use Monte Carlo method to estimate $\mu_i$ and $\sigma^2_i$ (the estimated values are $\hat{\mu}_i$ and $\hat{\sigma}_i^2$), and then use Algorithm \ref{Algorithm_Esti_U} to estimate the total variation distance  between $D_i$ and $\mathcal{N}(\hat{\mu}_i, \hat{\sigma}_i^2)$.

\begin{theorem}\label{Theorem_NE}
When $\sigma_i^2 \in [\sigma_{\min}^2, \sigma_{\max}^2]$, then we can set $n_H(\delta, \Delta) = {4\log {4\over \delta} \over \Delta^2}(\sqrt{2} + \sqrt{\sigma_{\max}^2 \over \pi} + \sqrt{4\sigma_{\max}^4 \over \pi\sigma_{\min}^4})^2 +  {128C \over \lambda^2\Delta^3} \log^2 {1024\beta^2 C\over \lambda^4\Delta^3}$ in the racing framework, and $\Delta_H(\delta,n) = \sqrt{2\log{4\over \delta}\over n}  + ({{8\sqrt{2}}C \log^2({\beta\over \lambda}\sqrt{8n\over \log{(2/\delta)}})\over n\lambda^2})^{1/3} + \sqrt{\sigma_{\max}^2\log{4 \over \delta} \over \pi n} + \sqrt{4\sigma_{\max}^4\log{4 \over \delta} \over\pi \sigma^4_{\min} n}$ in the LUCB framework.
\end{theorem}

\begin{proof}

Firstly, we have the following proposition and lemma (the proof of Lemma \ref{Lemma_NE} is referred to appendix).

\begin{proposition}\label{Proposition_NE}(Lemma A.2 in \cite{barrera2017cut})
For any $\mu_i, \hat{\mu}_i, \sigma^2$, $\Lambda_{\it TV}(\mathcal{N}(\mu_i, \sigma^2), \mathcal{N}(\hat{\mu}_i, \sigma^2)) \le {|\mu_i - \hat{\mu}_i| \over \sqrt{2\pi}}$.
\end{proposition}

\begin{restatable}{lemma}{LemmaOne}\label{Lemma_NE}
For any $\sigma_i^2, \hat{\sigma}_i^2 \in [\sigma_{\min}^2, \sigma_{\max}^2]$, $\Lambda_{\it TV}(\mathcal{N}(\hat{\mu}_i, \sigma^2_i), \mathcal{N}(\hat{\mu}_i, \hat{\sigma}^2_i)) \le {|\sigma^2_i - \hat{\sigma}^2_i| \over \sqrt{2\pi}\sigma^2_{\min} } $.
\end{restatable}


For $n$ number of observations, we know that with probability at least $1 - {\delta \over 4}$, $|\hat{\mu}_i - \mu_i| \le \sqrt{2\sigma_{\max}^2\log{4 \over \delta} \over n}$, and similarly with probability at least $1 - {\delta \over 4}$, $|\hat{\sigma}_i^2 - \sigma_i^2| \le  \sqrt{8\sigma_{\max}^4\log{4 \over \delta} \over n}$.

Proposition \ref{Proposition_NE} shows that $\Lambda_{\it TV}(\mathcal{N}(\mu_i, \sigma^2_i), \mathcal{N}(\hat{\mu}_i, \sigma^2_i)) \le {|\mu_i - \hat{\mu}_i| \over \sqrt{2\pi}}$, and Lemma \ref{Lemma_NE} shows that $\Lambda_{\it TV}(\mathcal{N}(\hat{\mu}_i, \sigma^2_i), \mathcal{N}(\hat{\mu}_i, \hat{\sigma}^2_i)) \le {|\sigma^2_i - \hat{\sigma}^2_i| \over \sqrt{2\pi}\sigma^2_{\min} }$. 

Thus, according to Theorem \ref{Theorem_CDUT},  setting $\Delta_H(\delta,n) = \sqrt{2\log{2\over \delta}\over n}  + ({{8\sqrt{2}}C \log^2({\beta\over \lambda}\sqrt{8n\over \log{(2/\delta)}})\over n\lambda^2})^{1/3} + \sqrt{\sigma_{\max}^2\log{4 \over \delta} \over \pi n} + \sqrt{4\sigma_{\max}^4\log{4 \over \delta} \over\pi \sigma^4_{\min} n} $ is enough to make sure that the error probability is less than or equal to $\delta$.

Similarly, we can obtain that $n_H(\delta, \Delta) = {4\log {4\over \delta} \over \Delta^2}(\sqrt{2} + \sqrt{\sigma_{\max}^2 \over \pi} + \sqrt{4\sigma_{\max}^4 \over \pi\sigma_{\min}^4})^2 +  {128C \over \lambda^2\Delta^3} \log^2 {1024\beta^2 C\over \lambda^4\Delta^3}$ is enough to make sure that the error probability is less than or equal to $\delta$. \hspace*{\fill}$\qed$

\end{proof}

\begin{mycorollary}\label{Proposition_4}
With $\Delta_H(\delta,n), n_H(\delta, \Delta)$ set as in Theorem \ref{Theorem_NE}, using Algorithm \ref{Algorithm_Esti_U} in the racing framework has time complexity $O(\sum_{i=1}^m {B^2\over \Delta_i^2} (\log{m\over \delta}+\log\log{1\over \Delta_i})) + O(\sum_{i=1}^m {B^3C \over \lambda^2\Delta_i^3} \log^2 {\beta BC\over \lambda\Delta_i}) $, and using Algorithm \ref{Algorithm_Esti_U} in LUCB framework has complexity $O(B^2h \log{Bh\over \delta}) + O(\sum_{i=1}^m {B^3C \over \lambda^2\Delta_i^3} \log^2 {\beta BC\over \lambda\Delta_i})$, where $O(\sum_{i=1}^m {B^3C \over \lambda^2\Delta_i^3} \log^2 {\beta BC\over \lambda\Delta_i})$ is a constant gap that does not depend on $\delta$ and $h = \sum_{i=1}^m {1\over \Delta_i^2}$.
\end{mycorollary}

Compare with prior works, using total variation distance can be more accurate than using Kolmogorov-Smirnov distance. This is because that $\Lambda_K(D,D') \le \Lambda_{\it TV}(D,D')$, i.e., a small total variation distance leads to a small Kolmogorov-Smirnov distance, but not vice versa.

\section{Future Work}

In this paper, we 
concentrate on dealing with the fixed-confidence problem. Therefore, a possible further research topic is the fixed-budget model (based on general reward functions).
In \cite{Audibert2010Best}, the authors use an elimination-based policy called SR
to solve the fixed-budget problem. We believe that it is not hard to
develop algorithms for fixed-budget problem based on the their algorithm framework and analysis in this paper.

Another interesting topic is to extend the pure exploration bandit problem with general reward functions to combinatorial setting, i.e., to find out a set of arms $S \subseteq A$ such that the joint distribution of arms in $S$ has the largest score (under a general reward function $H$). For example, let $R(S) = r(\bm{X},S)$ denote the random reward of a set of arms $S$, where $\bm{X}$ is a random vector that follows the fixed joint distribution, and our goal is to find out the best set of arms $S$ such that $R(S)$ has the largest $\tau$-quantile. This setting is common in online systems that follow some special combinatorial structures, e.g., recommendation websites and search engines. A further research about this topic can be really helpful for algorithm design in these applications. 

\bibliographystyle{icml2018}     
\bibliography{mlj}

\appendix

\section{Proofs of Propositions and Lemmas}

\subsection{Proof of Proposition \ref{Prop_Elim_F}}

{\PropOne*}

\begin{proof}

Let $\mathcal{E}_{R}$ be the event that 
\begin{equation*}
 \{\forall k, \forall i \in A(k), |H(D_i) - \hat{H}_i(k)| \le \rad(k)\},
\end{equation*}
where $A(k)$ is the set of remaining arms at the end of phase $k$. By definition of $n_H(\delta,\Delta)$

\begin{equation*}
\Pr[\mathcal{E}_{R}] \ge 1 - \sum_{i=1}^m \sum_{k=1}^\infty \delta(k) = 1 - \sum_{k=1}^\infty {\delta \over 2k^2} \ge 1 - \delta.
\end{equation*}

Then we prove that under event $\mathcal{E}_{R}$, Algorithm \ref{Algorithm_Elim_F} works correctly.

If arm $i^*$ is deleted, then $\exists k \ge 1, i \ne i^*$ such that $\hat{H}_{i^*}(k)\le \hat{H}_i(k) - 2\rad(k)$. However, under event $\mathcal{E}_{R}$, we always have $\hat{H}_{i^*}(k) \ge H(D_{i^*}) - \rad(k) > H(D_i) - \rad(k) \ge \hat{H}_i(k) - 2\rad(k)$, which makes a contradiction.


Next, we consider the complexity under $\mathcal{E}_{R}$. When $\rad(k) < {\Delta_i \over 4}$, we have $\hat{H}_{i^*}(k) -2\rad(k) \ge H(D_{i^*}) - 3\rad(k) = H(D_i) +\Delta_i -  3\rad(k) \ge \hat{H}_i(k) - 4\rad(k) + \Delta_i > \hat{H}_i(k)$, thus at the end of phase $k$, arm $i$ is deleted. 

This implies that any sub-optimal arm $i\ne i^*$ will be deleted when $\rad(k) < {\Delta_i \over 4}$. Thus, it can remain in $A$ at most for $\log {\Delta_i\over 8}$ rounds. 
Thus the time steps we used on it has upper bound $n_H\left({\delta \over 2m\log^2\left(8/\Delta_i\right)}, {\Delta_i \over 8}\right)$. As for the optimal arm $i^*$, we need to pull it until the last arm is eliminated. From the definition of $\Delta_{i^*}$, we know that it is upper bounded by $n_H\left({\delta \over 2m\log^2\left(8/\Delta_{i^*}\right)}, {\Delta_{i^*} \over 8}\right)$ as well.  \hspace*{\fill}$\qed$
\end{proof}

\subsection{Proof of Proposition \ref{Prop_LUCB_F}}

{\PropTwo*}

\begin{proof}
Let $\mathcal{E}_{L}$ be the event that 
\begin{equation*}
 \{\forall t, \forall i, |H(D_i) - \hat{H}_i(t)| \le \rad_i(t)\}.
\end{equation*}

By the definition of $\Delta_H(\delta,n)$, we have that 

\begin{equation*}
\Pr[\mathcal{E}_{L}] \ge 1 - \sum_{i=1}^m \sum_{t=1}^\infty {\delta \over 2mt^2} \ge  1 -  \sum_{t=1}^\infty {\delta \over 2t^2} \ge 1 - \delta.
\end{equation*}

Then we prove that under event $\mathcal{E}_{L}$, Algorithm \ref{Algorithm_LUCB_F} works correctly.

If the game stops with output not $i^*$, then $\exists t, i \ne i^*$ such that $\hat{H}_{i^*}(t) + \rad_{i^*}(t) \le \hat{H}_i(t) - \rad_i(t)$. However, under event $\mathcal{E}_{L}$, we always have $\hat{H}_{i^*}(t) + \rad_{i^*}(t) \ge H(D_{i^*}) > H(D_i)  \ge \hat{H}_i(t) - \rad_i(t)$, which makes a contradiction.


Now we come to the complexity under $\mathcal{E}_L$.
If the game does not stop at time $t$, then there are three possible cases: i) $i_1(t) = i$ and $i_2(t) = i^*$; ii) $i_1(t) = i^*$ and $ i_2(t) = i$; iii) $i_1(t) = i$ and $i_2(t) = j$.

In the first case, $i_1(t) = i$ implies that $\hat{H}_i(t) \ge \hat{H}_{i^*}(t)$, then we must have $H(D_i) + \rad_i(t) \ge \hat{H}_i(t) \ge \hat{H}_{i^*}(t) \ge H(D_{i^*})  - \rad_{i^*}(t)$. This means that $\rad_i(t) + \rad_{i^*}(t) \ge \Delta_i$. If we choose to pull arm $i$, then $2\rad_i(t) \ge \Delta_i$. Otherwise, if we choose to pull arm $i^*$, then $2\rad_{i^*}(t) \ge \Delta_i \ge \Delta_{i^*}$.

In the second case, $i_1(t) = i^*$ and $i_2(t) = i$ but the game does not stop at $t$ implies that $\hat{H}_i(t) + \rad_i(t) \ge \hat{H}_{i^*}(t) - \rad_{i^*}(t) $, then we must have $H(D_i) + 2\rad_i(t) \ge \hat{H}_i(t) + \rad_i(t) \ge \hat{H}_{i^*}(t) - \rad_{i^*}(t) \ge H(D_{i^*})  - 2\rad_{i^*}(t)$. This means that $2\rad_i(t) + 2\rad_{i^*}(t) \ge \Delta_i$. If we choose to pull arm $i$, then $4\rad_i(t) \ge \Delta_i$. Otherwise, if we choose to pull arm $i^*$, then $4\rad_{i^*}(t) \ge \Delta_i \ge \Delta_{i^*}$.

In the third case, $i_2(t) = j$ implies that $ \hat{H}_{j}(t) + \rad_j(t) \ge \hat{H}_{i^*}(t) + \rad_{i^*}(t)$, then we must have $H(D_j) + 2\rad_j(t)\ge \hat{H}_{j}(t) + \rad_j(t) \ge \hat{H}_{i^*}(t) + \rad_{i^*}(t) \ge H(D_{i^*})$. This means that if we choose to pull arm $j$, then $2\rad_{j}(t) \ge \Delta_j$. If we choose to pull arm $i$, then we must have $\rad_i(t) \ge \rad_j(t)$. Notice that $i_1(t) = i$ means that $\hat{H}_{i}(t) \ge \hat{H}_{j}(t)$. Then we know that $ \hat{H}_{i}(t) + \rad_i(t) \ge \hat{H}_{j}(t) + \rad_j(t) \ge \hat{H}_{i^*}(t) + \rad_{i^*}(t)$. By the same reason, we can obtain $2\rad_{i}(t) \ge \Delta_i$.



Thus we only pull arm $i$ when $\rad_i(t) \ge {\Delta_i \over 4}$. Under event $\mathcal{E}_{L}$, if $\forall i, \rad_i(t) < {\Delta_i \over 4}$, then we will output the correct arm $i^*$. This means that the complexity $T_L \le \min_{t > 0} \{t > \sum_{i=1}^m n_H({\delta \over 2mt^2}. {\Delta_i\over 4})\}$. \hspace*{\fill}$\qed$
\end{proof}

\subsection{Proof of Lemma \ref{Lemma_NE}}

{\LemmaOne*}

\begin{proof}
Note that $\Lambda_{\it TV}(\mathcal{N}(\hat{\mu}_i, \sigma^2_i), \mathcal{N}(\hat{\mu}_i, \hat{\sigma}^2_i)) = \Lambda_{\it TV}(\mathcal{N}(0, \sigma^2_i), \mathcal{N}(0, \hat{\sigma}^2_i))$.

Let $\Delta = |\sigma^2_i - \hat{\sigma}^2_i|$, and without loss of generality we suppose that $\sigma^2_i \le \hat{\sigma}^2_i$, then we can see that $\Lambda_{\it TV}(\mathcal{N}(0, \sigma^2_i), \mathcal{N}(0, \hat{\sigma}^2_i)) = \Lambda_{\it TV}(\mathcal{N}(0, 1), \mathcal{N}(0, 1 + {\Delta \over \sigma^2_i }))$, denote $s = 1 + {\Delta \over \sigma^2_i }$, and $f_s(x)$ the probability density function of distribution $\mathcal{N}(0, s)$, then according to the definition of total variation distance, $\Lambda_{\it TV}(\mathcal{N}(0, 1), \mathcal{N}(0, s)) = \int_{x: f_1(x) \ge f_s(x)} (f_1(x) - f_s(x)) dx$. 

It is easy to see that $\{x: f_1(x) \ge f_s(x)\} = [- g(s), g(s)]$, where $g(s) > 0$ satisfies that $f_1(g(s)) = f_s(g(s))$. Then from the symmetry of $f_1(x)$ and $f_s(x)$, one can see that  $\Lambda_{\it TV}(\mathcal{N}(0, 1), \mathcal{N}(0, s)) = 2\int_0^{g(s)} (f_1(x) - f_s(x)) dx$.

Note that 
\begin{eqnarray*}
\int_0^{g(s)} (f_1(x) - f_s(x)) dx &=& \int_0^{g(s)} f_1(x)  dx - \int_0^{g(s)} f_s(x) dx \\
&=& \int_0^{g(s)} f_1(x)  dx - \int_0^{g(s)/\sqrt{s}} f_1(x) dx \\
&=&  \int_{g(s)/\sqrt{s}} ^{g(s)} f_1(x)  dx \\
&\le& \int_{g(s)/\sqrt{s}} ^{g(s)} f_1(0)  dx \\
&\le& f_1(0) (g(s) - {g(s)/\sqrt{s}} ) \\
&=& {1\over \sqrt{2\pi}} \left( g(s){\sqrt{s} - 1 \over \sqrt{s}}\right).
\end{eqnarray*}

Now we come to bound $\left( g(s){\sqrt{s} - 1 \over \sqrt{s}}\right)$. Since $f_1(g(s)) = f_s(g(s))$, we have that ${1\over \sqrt{2\pi s}}\exp({-g^2(s) \over 2s}) = {1\over \sqrt{2\pi}}\exp({-g^2(s) \over 2})$. Thus $\exp({(s-1)g^2(s) \over s}) = s$, which means that $g^2(s) = {s\log s \over s-1}$.

Then 
\begin{eqnarray*}
\left( g(s){\sqrt{s} - 1 \over \sqrt{s}}\right)^2 &=& {s\log s \over s-1}{(\sqrt{s} - 1)^2 \over s} \\
&=& \log s{\sqrt{s} - 1 \over \sqrt{s} + 1} \\
&\le& (s-1){1+(s-1)/2 - 1 \over 2} \\
&=& {\Delta \over \sigma^2_i }\cdot {\Delta \over 4\sigma^2_i }\\
&\le& {\Delta^2 \over 4\sigma^4_{\min}}.
\end{eqnarray*}

Thus $\left( g(s){\sqrt{s} - 1 \over \sqrt{s}}\right) \le {\Delta \over 2\sigma^2_{\min} }$. This implies that $\Lambda_{\it TV}(\mathcal{N}(\hat{\mu}_i, \sigma^2_i), \mathcal{N}(\hat{\mu}_i, \hat{\sigma}^2_i)) \le 2 \cdot {1\over \sqrt{2\pi}}  \cdot {\Delta \over 2\sigma^2_{\min}} = {\Delta \over \sqrt{2\pi}\sigma^2_{\min} } $, which finish the proof of this lemma.  \hspace*{\fill}$\qed$
\end{proof}

\section{The DKW Inequality}

The Dvoretzky-Kiefer-Wolfowitz inequality is used to bound the Kolmogorov-Smirnov distance between the empirical distribution and the real distribution, which is first proposed by \cite{dvoretzky1956asymptotic} and then improved by \cite{massart1990tight}. 

\begin{proposition}
(Dvoretzky-Kiefer-Wolfowitz inequality) Let $X_1, X_2, \cdots, X_n$ be  i.i.d. random variables with cumulative distribution function $F$, and let $F_n$ denote the cumulative distribution function of the empirical distribution, i.e., $F_n(x) = {1\over n}\sum_{i=1}^n \I[X_i \le x]$, then
\begin{equation*}
\Pr\left[\sup_x |F_n(x) - F(x)| \ge \epsilon \right] \le 2e^{-2n\epsilon^2}.
\end{equation*}
\end{proposition}

%
%

\end{document}